\documentclass[10pt,reqno]{amsart}

\usepackage[T1]{fontenc}
\usepackage[utf8]{inputenc}
\usepackage{microtype}
\usepackage{mathtools,amssymb,amsthm,mathrsfs,bm}
\usepackage{enumitem}
\usepackage{csquotes}
\usepackage[colorlinks=true,linkcolor=blue,citecolor=blue,urlcolor=blue]{hyperref}
\usepackage[
  backend=biber,
  style=numeric,
  sorting=nyt,
  maxbibnames=99
]{biblatex}
\usepackage{aliascnt}
\usepackage{silence}
\allowdisplaybreaks
\numberwithin{equation}{section}

\theoremstyle{plain}
\newtheorem{theorem}{Theorem}[section]
\newaliascnt{proposition}{theorem}
\newtheorem{proposition}[proposition]{Proposition}
\aliascntresetthe{proposition}
\newaliascnt{lemma}{theorem}
\newtheorem{lemma}[lemma]{Lemma}
\aliascntresetthe{lemma}
\newaliascnt{corollary}{theorem}
\newtheorem{corollary}[corollary]{Corollary}
\aliascntresetthe{corollary}

\theoremstyle{definition}
\newaliascnt{definition}{theorem}

\aliascntresetthe{definition}
\newaliascnt{remark}{theorem}
\newtheorem{remark}[remark]{Remark}
\aliascntresetthe{remark}
\newaliascnt{example}{theorem}

\aliascntresetthe{example}

\usepackage[nameinlink,noabbrev]{cleveref}

\crefname{theorem}{theorem}{theorems}
\Crefname{theorem}{Theorem}{Theorems}
\crefname{proposition}{proposition}{propositions}
\Crefname{proposition}{Proposition}{Propositions}
\crefname{lemma}{lemma}{lemmas}
\Crefname{lemma}{Lemma}{Lemmas}
\crefname{corollary}{corollary}{corollaries}
\Crefname{corollary}{Corollary}{Corollaries}
\crefname{definition}{definition}{definitions}
\Crefname{definition}{Definition}{Definitions}
\crefname{remark}{remark}{remarks}
\Crefname{remark}{Remark}{Remarks}
\crefname{example}{example}{examples}
\Crefname{example}{Example}{Examples}
\Crefname{section}{Section}{Sections}
\Crefname{equation}{equation}{equations}

\newcommand{\R}{\mathbb{R}}
\newcommand{\C}{\mathbb{C}}
\newcommand{\HH}{\mathbb{H}}
\newcommand{\F}{\mathbb{F}}
\newcommand{\Id}{I}
\newcommand{\tr}{\operatorname{Tr}}
\newcommand{\re}{\operatorname{Re}}
\newcommand{\diag}{\operatorname{diag}}
\newcommand{\vol}{\operatorname{vol}}
\newcommand{\van}{\operatorname{van}}

\newcommand{\GL}{GL}
\newcommand{\U}{U}
\newcommand{\Orth}{O}
\newcommand{\Sp}{Sp}
\newcommand{\Md}[1]{{M}_{#1}}
\newcommand{\dd}{{d}}
\newcommand{\eps}{\varepsilon}

\newcommand{\set}[1]{\left\{#1\right\}}
\newcommand{\abs}[1]{\left|#1\right|}

\newcommand{\inner}[2]{\left\langle #1,#2\right\rangle}

\newcommand{\OO}{\mathcal{O}}
\newcommand{\MM}{\mathcal{M}}
\newcommand{\Fib}{\mathcal{F}}

\newcommand{\yy}{\mathfrak{y}}

\usepackage[mathlines]{lineno}
\modulolinenumbers[1] 

\title[Entropy Formula for Real, Complex, and Quaternionic DLNs]{On the Entropy Formula for Real, Complex, and Quaternionic Deep Linear Networks}

\author{Luis Contreras}
\address{Departamento de Matem\'aticas, CINVESTAV-IPN, Av. Instituto Polit\'ecnico Nacional 2508, Col. San Pedro Zacatenco, 07360 Ciudad de M\'exico, M\'exico}
\email{lcontrerasm@math.cinvestav.mx}

\author{Marco Nahas}
\address{Division of Applied Mathematics, Brown University, Providence, RI 02912, USA}
\email{marco\_nahas@brown.edu}

\author{Tejas Kotwal}
\address{Division of Applied Mathematics, Brown University, Providence, RI 02912, USA}
\email{tejas\_suresh\_kotwal@alumni.brown.edu}



\thanks{This work was partially supported by NSF grant DMS 2407055.}

\begin{document}

\begin{abstract}
We extend the entropy formula of Menon and Yu for the real Deep Linear
Network (DLN) to its complex and quaternionic analogues, obtaining a unified
formula for DLNs over \(\R\), \(\C\), and \(\HH\).
\end{abstract}

\maketitle

\section{Overview}

In this paper, we compute the Boltzmann entropy of functions represented by
deep linear networks (DLNs) over $\mathbb R$, $\mathbb C$, and $\mathbb H$.
The DLN is a basic model for the geometry of overparametrization in deep
learning, with a precise geometric and thermodynamic description of its
training dynamics \cite{MenonGeometry2024}.

Given depth \(N\ge 2\) and width \(d\ge 1\), a DLN over \(\F\)
is described by the multiplication map
\[
\phi:\Md{d}(\F)^N\longrightarrow \Md{d}(\F),\qquad
\phi(W_N,\ldots,W_1)=W_N\cdots W_1 .
\]
The space of parameters is \(\Md{d}(\F)^N\) while the space of observables
is \(\Md{d}(\F)\). For a parameter \(\mathbf{W}\), we write
\[
\mathbf{W}=(W_N,\ldots,W_1)\in \Md{d}(\F)^N,\qquad
X=\phi(\mathbf{W})=W_N\cdots W_1,
\]
and call \(X\) the end-to-end matrix of \(\mathbf{W}\). Here
\(\F\in\{\R,\C,\HH\}\), where \(\R\), \(\C\), and \(\HH\) denote
the real numbers, complex numbers, and quaternions, respectively.
We write \(\Md{d}(\F)\) for the space of \(d\times d\) matrices over \(\F\),
with \(\GL_d(\F)\subset \Md{d}(\F)\) denoting the invertible matrices.
We also set \(\beta:=\dim_{\R}\F\), so \(\beta=1,2,4\) for
\(\F=\R,\C,\HH\), respectively.

Given \(X\in \GL_d(\F)\), we consider the set of
\emph{balanced factorizations} of \(X\), given by
\[
\OO_X^{\beta}
:=
\left\{
\mathbf{W} \in \GL_d(\F)^N :
W_N\cdots W_1=X,\quad
W_{p+1}^*W_{p+1}=W_pW_p^*
\ \text{for } 1\le p\le N-1
\right\}.
\]
The Boltzmann entropy of \(X\) is the logarithm of the volume,
\[
S^{\beta}(X):=\log\vol(\OO_X^{\beta}),
\]
where the volume is taken with respect to the metric induced from the real
trace pairing on \(\Md{d}(\F)^N\).
Our main result (\Cref{thm:intro-entropy}) is an explicit formula for the entropy, generalizing Menon and Yu's result \cite{MenonYu2025}
from the case $\beta=1$ to the cases $\beta=1, 2, 4$.

\subsection{Main result}

Let \(\Orth_d\), \(\U_d\), and \(\Sp_d\) denote the orthogonal, unitary,
and compact symplectic groups, respectively. Collectively, we write
\(K_1=\Orth_d\), \(K_2=\U_d\), and \(K_4=\Sp_d\), and let
\(c_{\beta}:=\vol(K_{\beta})\), where the volume is taken with respect to
the standard bi-invariant metric on \(K_{\beta}\).

We denote by \(\van(D)\) the Vandermonde determinant of \(D\). 
For a diagonal matrix
\(
D=\diag(\lambda_1,\ldots,\lambda_d),
\)
with real diagonal entries
\(
\lambda_1\ge \lambda_2\ge \cdots \ge \lambda_d,
\)
set
\[
\van(D):=\prod_{1\le s<l\le d}(\lambda_s-\lambda_l).
\]

\begin{theorem}[Entropy formula]\label{thm:intro-entropy}
Let \(X\in\GL_d(\F)\) have distinct singular values
\(
\sigma_1>\cdots>\sigma_d>0,
\)
and write \(\Sigma=\diag(\sigma_1,\ldots,\sigma_d)\). Then
\begin{equation}\label{eq:entropy-formula-vandermonde}
S^{\beta}(X)
=
(N-1)\log c_{\beta}
+\frac{\beta}{2}\log\frac{\van(\Sigma^2)}{\van(\Sigma^{2/N})}
+(\beta-1)
\left(
\frac d2\log N+
\log\frac{\det\Sigma}{\det(\Sigma^{1/N})}
\right).
\end{equation}
Equivalently,
\begin{equation}\label{eq:entropy-formula-entrywise}
\begin{aligned}
S^{\beta}(X)
={}&
(N-1)\log c_{\beta}
+\frac{\beta}{2}\sum_{1\le s<l\le d}
\log\left(
\frac{\sigma_s^2-\sigma_l^2}
{\sigma_s^{2/N}-\sigma_l^{2/N}}
\right)\\
&+(\beta-1)
\left(
\frac d2\log N+
\left(1-\frac1N\right)\sum_{s=1}^d\log\sigma_s
\right).
\end{aligned}
\end{equation}
\end{theorem}

The assumption of distinct singular values is used to identify the orbit
\(\OO_X^\beta\) with \(K_\beta^{N-1}\) via \Cref{prop:intro-orbit}.
The right-hand side of \eqref{eq:entropy-formula-entrywise}, however, extends
real-analytically through collisions of singular values \cite[Theorem~1.8]{ChenKotwalMenon2025}. Indeed, for
\(1\le s<l\le d\), if \(a=\sigma_s^{2/N}\) and \(b=\sigma_l^{2/N}\), then
\begin{equation}\label{eq:intro-finite-sum-extension}
\frac{\sigma_s^2-\sigma_l^2}
     {\sigma_s^{2/N}-\sigma_l^{2/N}}
=
\frac{a^N-b^N}{a-b}
=
\sum_{m=0}^{N-1}a^{N-1-m}b^m,
\end{equation}
a positive polynomial for \(a,b>0\). Throughout the paper, quotients of this
form are understood by \eqref{eq:intro-finite-sum-extension} when singular values collide.

\begin{remark}[Analogy with RMT]
The parameter \(\beta=\dim_{\R}\F\) plays a role analogous to the Dyson index in
random matrix theory (RMT). In the classical Gaussian ensembles \cite{mehta2004random}, the volumes of
isospectral orbits carry Vandermonde factors with exponents \(\beta=1,2,4\)
\cite{HuangInauenMenon2023}. The entropy formula
\eqref{eq:entropy-formula-entrywise} contains an analogous ratio of
Vandermondes,
\[
\prod_{1\le s<l\le d}
\left(
\frac{\sigma_s^2-\sigma_l^2}
     {\sigma_s^{2/N}-\sigma_l^{2/N}}
\right)^{\beta/2}.
\]
An important difference is worth emphasizing. By
\eqref{eq:intro-finite-sum-extension}, the entropy term exhibits \emph{no} singular-value repulsion, in contrast with the usual logarithmic repulsion in Dyson Brownian motion. Unlike the \(\beta\)-ensembles of Dumitriu and Edelman
\cite{dumitriu2002matrix}, we do not presently have a tridiagonal matrix model for the entropy formula at arbitrary \(\beta>0\).
\end{remark}

For \(\F=\R\), \Cref{thm:intro-entropy} recovers the entropy formula
for the real DLN \cite[Theorem~4]{MenonYu2025}.  Menon and Yu's proof
proceeds by diagonalizing certain tridiagonal Jacobi matrices using the
theory of Chebyshev polynomials.  The same ideas also yield an orthonormal
basis for the tangent space of the balanced manifold.  The resulting basis
allows them to prove that the multiplication map \(\phi\) restricts to a
Riemannian submersion from the balanced manifold in parameter space onto
the space of end-to-end matrices \cite[Theorems~8,~11]{MenonYu2025}.
In this paper, we do not compute the full orthonormal basis or pursue the
corresponding Riemannian submersion for the complex and quaternionic DLNs.
For the K\"ahler reduction of the complex DLN, see
\cite[Chapter~4]{Kotwal2026}.

\subsection{The DLN metric and a determinant identity}
We recall the DLN metric \(g_N\) on \(\GL_d(\F)\), the space of
invertible end-to-end matrices, following \cite{MenonYu2025,bah2022learning}. For \(X\in \GL_d(\F)\), define
the real-linear operator
\[
\mathcal A_{N,X}(Z)
=
\sum_{p=1}^N
(XX^*)^{(N-p)/N}\,Z\,(X^*X)^{(p-1)/N},
\qquad Z\in \Md{d}(\F).
\]
The powers are defined by functional calculus for positive Hermitian matrices.
We regard \(\Md{d}(\F)\) as a real inner product space with the pairing
\(\langle Y,Z\rangle:=\re\tr(Y^*Z)\).  With respect to this pairing,
\(\mathcal A_{N,X}\) is positive.  Under the canonical identification
\(T_X\GL_d(\F)=\Md{d}(\F)\), we define the DLN metric by
\begin{equation}\label{eq:intro-dln-metric}
g_{N,X}(Z_1,Z_2)
=
\re\tr\left(Z_1^*\,\mathcal A_{N,X}^{-1}(Z_2)\right),
\qquad
Z_1,Z_2\in \Md{d}(\F).
\end{equation}
All tangent spaces and determinants below are understood over the underlying
real vector spaces.  In the quaternionic case, the same formulas are justified by
the cyclicity of \(\re\tr\).

\begin{proposition}\label{prop:intro-operator}
Let \(X\in \GL_d(\F)\) have singular value decomposition
\(X=U_N\Sigma U_0^*\), and assume that the singular values of \(X\) are
distinct.  Let \(\beta=\dim_{\R}\F\).  Then
\begin{equation}\label{eq:division-volume-determinant}
\vol(\OO_X^{\beta})\,
\det_{\R}(\mathcal A_{N,X}^{-1})^{1/4}
=
c_{\beta}^{N-1}
N^{(\beta-2)d/4}
\left(
\frac{\det\Sigma}{\det(\Sigma^{1/N})}
\right)^{(\beta-2)/2}.
\end{equation}
Equivalently,
\begin{equation}\label{eq:division-entropy-determinant}
S^{\beta}(X)
=
(N-1)\log c_{\beta}
-\frac14\log\det_{\R}(\mathcal A_{N,X}^{-1})
+\frac{\beta-2}{4}d\log N
+\frac{\beta-2}{2}\log\frac{\det\Sigma}{\det(\Sigma^{1/N})}.
\end{equation}
\end{proposition}

When \(\F=\C\), we have \(\beta=2\), and
\eqref{eq:division-volume-determinant} reduces to
\begin{equation}\label{eq:complex-volume-determinant}
\vol(\OO_X^{2})\,
\det_{\R}(\mathcal A_{N,X}^{-1})^{1/4}
=
c_{2}^{N-1}.
\end{equation}

\begin{remark}[Infinite depth and renormalized entropy]
\label{rem:infinite-depth-renormalized-entropy}
Consider the renormalized operator
\[
  \frac{1}{N}\mathcal A_{N,X}(Z)
  \longrightarrow
  \mathcal A_{\infty,X}(Z)
  :=
  \int_0^1
  (XX^*)^{1-t} Z (X^*X)^t\,\dd t,
\]
which defines the metric
\(g_{\infty,X}(Z_1,Z_2)
  :=
  \re\tr\left(
  Z_1^*\mathcal A_{\infty,X}^{-1}(Z_2)
  \right)\)
for the DLN in the infinite-depth limit \cite{CohenMenonVeraszto2023}.

The limiting entropy can be obtained from \Cref{thm:intro-entropy} by
subtracting terms that diverge as \(N\to\infty\).  
The same renormalization applied to \Cref{prop:intro-operator}
provides another expression for the entropy formula at infinite depth.
Passing to the limit gives
\[
  S^\beta_{\infty}(X)
  =
  -\frac14\log\det_{\R}(\mathcal A_{\infty,X}^{-1})
  +\frac{\beta-2}{2}\log\det\Sigma .
\]
Evaluating the determinant in terms of the singular values
\(\sigma_1,\ldots,\sigma_d\) of \(X\) gives
\[
  S^\beta_{\infty}(X)
  =
  (\beta-1)\sum_{s=1}^d \log\sigma_s
  +
  \frac{\beta}{2}
  \sum_{1\le s<l\le d}
  \log\!\left(
  \frac{\sigma_s^2-\sigma_l^2}
       {\log\sigma_s^2-\log\sigma_l^2}
  \right).
\]
As in the formula at finite depth, this expression is understood by continuity at
repeated singular values.
\end{remark}

\subsection{Organization of the paper}

In \Cref{sec:setup}, we set notation for matrices over \(\R\), \(\C\), and
\(\HH\), and introduce the balanced manifold and the orbit
\(\OO_X^{\beta}\).  In \Cref{sec:orbit}, we prove
\Cref{prop:intro-orbit}, identifying \(\OO_X^{\beta}\) with
\(K_{\beta}^{N-1}\) for matrices with distinct singular values.  In
\Cref{sec:vertical}, we compute the metric induced on this orbit by the
ambient real trace pairing.  In \Cref{sec:entropy}, we evaluate the resulting
block determinants and prove \Cref{thm:intro-entropy} and
\Cref{prop:intro-operator}.
\section{Preliminaries}\label{sec:setup}

\subsection*{Matrix conventions and inner products}
For \(A=(a_{sl})\in \Md{d}(\F)\), set
\(
A^*=(\overline{a_{ls}}).
\)
Thus \(A^*\) is the transpose over \(\R\), the conjugate transpose over \(\C\),
and the quaternionic adjoint over \(\HH\).

All inner products are real. We use the pairing
\(
\inner{A}{B}:=\re\tr(A^*B),
\)
where $A,B\in\Md{d}(\F)$.
For quaternionic matrices, we have the cyclic property \(\re\tr(AB)=\re\tr(BA)\) whenever both products are
defined.

We use the singular value decomposition (SVD) over \(\R,\C,\HH\). Thus every
\(X\in\Md{d}(\F)\) admits a decomposition
\[
X=U\Sigma V^*,
\qquad
\Sigma=\diag(\sigma_1,\dots,\sigma_d),
\qquad
\sigma_1\ge\cdots\ge\sigma_d\ge 0,
\]
with \(U,V\in\Orth_d\), \(\U_d\), or \(\Sp_d\), respectively, when
\(\F=\R,\C\) or $\HH$. In the quaternionic case this follows from the spectral theorem
for quaternionic Hermitian matrices \cite[Theorem~7.2]{Zhang1997}.

For $\mathbf A, \mathbf B \in \Md{d}(\F)^N$ where $\mathbf A=(A_N,\dots,A_1), \mathbf B=(B_N,\dots,B_1)$
set
\begin{equation}\label{eq:ambient-inner-product}
\inner{\mathbf A}{\mathbf B}
:=
\sum_{p=1}^N \re\tr(A_p^*B_p).
\end{equation}
This inner product defines the Euclidean metric on \(\Md{d}(\F)^N\).

\subsection*{Fibres and balanced factorizations}
Recall the multiplication map
\[
\phi:\Md{d}(\F)^N\to\Md{d}(\F),
\qquad
\phi(\mathbf W)=W_N\cdots W_1 .
\]
For \(X\in\GL_d(\F)\), the fibre over $X$ is
\[
\Fib_X^\beta:=\phi^{-1}(X).
\]
Since \(X\) is invertible, every element of \(\Fib_X^\beta\) has invertible layers.

The equations for balancedness are
\[
W_{p+1}^*W_{p+1}=W_pW_p^*,
\qquad 1\le p\le N-1.
\]
Together, they define the balanced variety
\[
\MM_{\mathbf{0}}^\beta
:=
\set{
\mathbf W\in\Md{d}(\F)^N:
W_{p+1}^*W_{p+1}=W_pW_p^*
\text{ for }1\le p\le N-1
}.
\]
On \(\GL_d(\F)^N\), the same equations define the balanced manifold
\[
\MM^\beta:=\MM_{\mathbf{0}}^\beta\cap\GL_d(\F)^N .
\]
Hence, the set of balanced factorizations of \(X\) is
\[
\OO_X^\beta:=\Fib_X^\beta\cap\MM^\beta .
\]
When \(X\) has distinct singular values, by \Cref{prop:intro-orbit}, we have that \(\OO_X^\beta\) is a compact smooth submanifold of
\(\Md{d}(\F)^N\). Its volume is always taken with respect to the metric induced by
\eqref{eq:ambient-inner-product}.

\subsection*{Group actions}
Recall that \(K_\beta\) denotes \(\Orth_d\), \(\U_d\), or \(\Sp_d\) according as
\(\beta=1,2,4\), respectively, and write its Lie algebra as
\[
\mathfrak{k}_\beta
:=
\set{A\in\Md{d}(\F):A^*=-A}.
\]
We equip \(\mathfrak{k}_\beta\) with the real inner product
\begin{equation}\label{eq:lie-inner-product}
\inner{A}{B}_{\mathfrak{k}}
:=
\re\tr(A^*B).
\end{equation}

The group \(K_\beta^{N-1}\) acts on \(\Md{d}(\F)^N\) by
\begin{equation}\label{eq:compact-action}
(U_{N-1},\dots,U_1)\cdot(W_N,\dots,W_1)
:=
\bigl(
W_NU_{N-1}^*,
U_{N-1}W_{N-1}U_{N-2}^*,
\dots,
U_1W_1
\bigr).
\end{equation}
The action is isometric, preserves the end-to-end product, and preserves the equations for
balancedness. Indeed, the intermediate factors cancel in the product, and
for \(1\le p\le N-1\),
\[
\widetilde W_{p+1}^*\widetilde W_{p+1}
=
U_pW_{p+1}^*W_{p+1}U_p^*,
\qquad
\widetilde W_p\widetilde W_p^*
=
U_pW_pW_p^*U_p^* .
\]
Hence each \(\OO_X^\beta\) is invariant under \eqref{eq:compact-action}.

We also use the action of \(K_\beta\times K_\beta\) on \(\Md{d}(\F)^N\) given by
\begin{equation}\label{eq:endpoint-action}
(Q,R)\cdot(W_N,\dots,W_1)
:=
(QW_N,W_{N-1},\dots,W_2,W_1R^*).
\end{equation}
This action is also isometric.

\begin{proposition}\label{prop:unitary-invariance}
Let \(X\in\GL_d(\F)\) have distinct singular values, and let \(Q,R\in K_\beta\).
The endpoint action \eqref{eq:endpoint-action} restricts to an isometry
\[
\Psi_{Q,R}:\OO_X^\beta\longrightarrow\OO_{QXR^*}^\beta,
\qquad
\Psi_{Q,R}(\mathbf W)=(Q,R)\cdot\mathbf W .
\]
In particular, if \(X=U_N\Sigma U_0^*\) is a singular value decomposition, then
\[
S^\beta(X)=S^\beta(\Sigma).
\]
\end{proposition}

\begin{proof}
Let \(\mathbf W=(W_N,\dots,W_1)\in\Fib_X^{\beta}\). Then
\[
\phi\bigl((Q,R)\cdot\mathbf W\bigr)
=
QW_N\cdots W_1R^*
=
QXR^*.
\]
Thus the endpoint action sends \(\Fib_X^{\beta}\) to \(\Fib_{QXR^*}^{\beta}\).

The only equations that can change are those with \(p=1\) or
\(p=N-1\). For these equations,
\[
(QW_N)^*(QW_N)=W_N^*W_N,
\qquad
(W_1R^*)(W_1R^*)^*=W_1W_1^*.
\]
Hence the equations for balancedness are preserved. The action also preserves
invertibility, so it sends \(\MM^{\beta}\) to itself. Therefore it restricts to a map
\[
\OO_X^{\beta}\longrightarrow\OO_{QXR^*}^{\beta}.
\]
Its inverse is the restriction of the endpoint action by \((Q^*,R^*)\). Since the
endpoint action preserves \eqref{eq:ambient-inner-product}, the restricted map is an
isometry.

Taking \(Q=U_N^*\) and \(R=U_0^*\) gives an isometry
\[
\OO_X^{\beta}\longrightarrow\OO_{\Sigma}^{\beta}.
\]
The volumes are equal, and the entropy identity follows.
\end{proof}

\subsection*{Parametrization of \texorpdfstring{\(\OO_X^{\beta}\)}{OXbeta}}

Choose and fix a singular value decomposition
\(
X=U_N\Sigma U_0^*,
\)
and set \(\Lambda=\Sigma^{1/N}\), where the positive $N$-th root is taken.
The center associated with this chosen SVD is the balanced factorization obtained
by distributing the singular values evenly across the layers
\[
\mathbf C_X=(U_N\Lambda,\Lambda,\ldots,\Lambda U_0^*).
\]

\begin{proposition}\label{prop:intro-orbit}
Let \(X\in\GL_d(\F)\) have distinct singular values. Then the map
\begin{equation}\label{eq:orbit-map}
\yy_X:K_{\beta}^{N-1}\to \OO_X^{\beta}
\end{equation}
defined by
\[
\yy_X(U_{N-1},\ldots,U_1)
=
\bigl(
U_N\Lambda U_{N-1}^*,\,
U_{N-1}\Lambda U_{N-2}^*,\,
\ldots,\,
U_1\Lambda U_0^*
\bigr)
\]
is a diffeomorphism. In particular,
\(
\OO_X^{\beta}=K_{\beta}^{N-1}\cdot \mathbf C_X.
\)
\end{proposition}

The product of the displayed factors telescopes to \(X\), and the equations for balancedness are immediate. Hence \(\yy_X\) is well defined as a map into
\(\OO_X^\beta\). The proof that \(\yy_X\) is a diffeomorphism is presented in \Cref{sec:orbit}.

The map \(\yy_X\) parametrizes the orbit \(\OO_X^\beta\) for a fixed end-to-end matrix \(X\). One could similarly parametrize the full balanced manifold \(\mathcal M^\beta\), allowing \(X\) to vary, but here we restrict attention to the orbit map \(\yy_X\). In the real case, the Riemannian geometry of \(\mathcal M^\beta\) is developed in \cite{MenonYu2025}.

\section{Proof of Proposition 2.2}\label{sec:orbit}

We use the ordered SVD fixed in the setup for
\Cref{prop:intro-orbit},
\[
X=U_N\Sigma U_0^*,
\qquad
\Sigma=\diag(\sigma_1,\dots,\sigma_d),
\qquad
\sigma_1>\cdots>\sigma_d>0,
\]
and set \(\Lambda=\Sigma^{1/N}\). Since the singular values are distinct and
ordered, the diagonal factor \(\Sigma\) is determined by \(X\). The only remaining
ambiguity in the singular-vector factors is simultaneous multiplication on the
right by the same diagonal unit.
We therefore introduce the diagonal stabilizer
\[
\Delta_{\beta}:=Z_{K_{\beta}}(\Sigma)
=
\{t\in K_{\beta}:t\Sigma t^*=\Sigma\}.
\]
Since the singular values are distinct, this stabilizer is explicitly
\[
\Delta_{\beta}
=
\begin{cases}
\set{\diag(\eps_1,\dots,\eps_d):\eps_s=\pm1}, & \F=\R,\\
\set{\diag(z_1,\dots,z_d):|z_s|=1}, & \F=\C,\\
\set{\diag(q_1,\dots,q_d):q_s\in\Sp_1}, & \F=\HH.
\end{cases}
\]
Indeed, every \(t\in\Delta_{\beta}\) gives the same singular value decomposition
\[
X=(U_Nt)\Sigma(U_0t)^*,
\]
since \(t\Sigma t^*=\Sigma\). Conversely, as shown in
\Cref{lem:svd-uniqueness} below, every ordered SVD of
\(X\) with diagonal factor \(\Sigma\) arises uniquely in this way.

\begin{lemma}\label{lem:balanced-normal-form}
Let \(\mathbf W=(W_N,\dots,W_1)\in \MM^{\beta}\). Then there exist matrices
\(
Q_0,\dots,Q_N\in K_{\beta}
\)
and a positive diagonal matrix \(\Lambda_{\mathbf W}\), with diagonal entries in
nonincreasing order, such that
\[
W_p=Q_p\Lambda_{\mathbf W}Q_{p-1}^*,
\qquad
1\le p\le N.
\]
\end{lemma}

\begin{proof}
Choose a singular value decomposition of the first layer,
\(
W_1=Q_1\Lambda_{\mathbf W}Q_0^*,
\)
where \(\Lambda_{\mathbf W}\) is positive diagonal with diagonal entries in
nonincreasing order. Suppose inductively that
\(
W_p=Q_p\Lambda_{\mathbf W}Q_{p-1}^*
\)
has been constructed for some \(1\le p\le N-1\). Then
\[
W_pW_p^*=Q_p\Lambda_{\mathbf W}^2Q_p^*.
\]
Since \(\mathbf W\) is balanced,
\[
W_{p+1}^*W_{p+1}=W_pW_p^*
=Q_p\Lambda_{\mathbf W}^2Q_p^*.
\]
Define
\[
Q_{p+1}:=W_{p+1}Q_p\Lambda_{\mathbf W}^{-1}.
\]
Because \(\mathbf W\in\MM^{\beta}\), all layers are invertible, so this is well defined.
Moreover,
\[
Q_{p+1}^*Q_{p+1}
=
\Lambda_{\mathbf W}^{-1}Q_p^*W_{p+1}^*W_{p+1}Q_p
\Lambda_{\mathbf W}^{-1}
=
I.
\]
Thus \(Q_{p+1}\in K_{\beta}\), and
\(
W_{p+1}=Q_{p+1}\Lambda_{\mathbf W}Q_p^*.
\)
The claim follows by induction.
\end{proof}

The next lemma is the only point in the proof where distinct singular values are
used.

\begin{lemma}\label{lem:svd-uniqueness}
Assume that \(\Sigma=\diag(\sigma_1,\dots,\sigma_d)\) has distinct positive entries.
If
\[
Q\Sigma R^*=\widetilde Q\Sigma \widetilde R^*,
\qquad
Q,R,\widetilde Q,\widetilde R\in K_{\beta},
\]
then there exists a unique \(t\in \Delta_{\beta}\) such that
\[
\widetilde Q=Qt,
\qquad
\widetilde R=Rt.
\]
\end{lemma}

\begin{proof}
Set
\[
A:=Q^*\widetilde Q,
\qquad
B:=R^*\widetilde R.
\]
Then \(A,B\in K_{\beta}\) and
\(
\Sigma=A\Sigma B^*.
\)
Multiplying by the adjoint gives
\[
\Sigma^2=A\Sigma^2A^*,
\qquad
\Sigma^2=B\Sigma^2B^*.
\]
Hence \(A\) and \(B\) commute with \(\Sigma^2\). Since \(\Sigma^2\) is diagonal
with distinct real diagonal entries, both \(A\) and \(B\) are diagonal. Indeed, for
\(s\ne l\),
\[
(\sigma_s^2-\sigma_l^2)A_{sl}=0,
\qquad
(\sigma_s^2-\sigma_l^2)B_{sl}=0.
\]
Thus
\[
A=\diag(a_1,\dots,a_d),
\qquad
B=\diag(b_1,\dots,b_d),
\]
with \(\abs{a_s}=\abs{b_s}=1\) for each \(s\). Returning to
\(\Sigma=A\Sigma B^*\) and comparing diagonal entries gives
\(
a_s\sigma_s\overline{b_s}=\sigma_s,
\)
for $1\le s\le d.$
Since \(\sigma_s>0\), we obtain \(a_s=b_s\) for every \(s\). Therefore
\(A=B=t\) for some \(t\in\Delta_{\beta}\), and the stated formulas follow.
Uniqueness follows from \(t=Q^*\widetilde Q\).
\end{proof}

\begin{proof}[Proof of \Cref{prop:intro-orbit}]
Viewed as a map to \(\Md{d}(\F)^N\), \(\yy_X\) is smooth by construction.
Its image lies in \(\OO_X^{\beta}\) because the product of the factors telescopes to
\(X\), and the equations for balancedness are immediate from the definition.

To prove surjectivity, let
\(
\mathbf W=(W_N,\dots,W_1)\in \OO_X^{\beta}.
\)
By \Cref{lem:balanced-normal-form}, there exist \(Q_0,\dots,Q_N\in K_{\beta}\) and
a positive diagonal matrix \(\Lambda_{\mathbf W}\), with diagonal entries in
nonincreasing order, such that
\(
W_p=Q_p\Lambda_{\mathbf W}Q_{p-1}^*,
\)
for $1\le p\le N.$
Multiplying the factors gives
\(
X=Q_N\Lambda_{\mathbf W}^NQ_0^*.
\)
Since \(\Lambda_{\mathbf W}^N\) is positive diagonal with entries in nonincreasing
order, this is a singular value decomposition of \(X\) with ordered singular values.
Comparing with $X=U_N\Sigma U_0^*$, whose singular values are strictly decreasing,
we obtain
\(
\Lambda_{\mathbf W}^N=\Sigma,
\)
and $\Lambda_{\mathbf W}=\Lambda.$
Now compare the two singular value decompositions
\[
X=U_N\Sigma U_0^*=Q_N\Sigma Q_0^*.
\]
By \Cref{lem:svd-uniqueness}, there exists \(t\in\Delta_{\beta}\) such that
\[
Q_N=U_Nt,
\qquad
Q_0=U_0t.
\]
Set
\[
V_p:=Q_pt^*,
\qquad
1\le p\le N-1.
\]
Since \(t\) commutes with \(\Lambda\), we have
\[
W_p=Q_p\Lambda Q_{p-1}^*
=V_p\Lambda V_{p-1}^*,
\qquad
2\le p\le N-1,
\]
and also
\[
W_N=U_N\Lambda V_{N-1}^*,
\qquad
W_1=V_1\Lambda U_0^*.
\]
Thus
\(
\mathbf W=\yy_X(V_{N-1},\dots,V_1),
\)
so \(\yy_X\) is surjective.

To prove injectivity, suppose
\[
\yy_X(A_{N-1},\dots,A_1)=\yy_X(B_{N-1},\dots,B_1).
\]
Comparing the first layer gives
\(
A_1\Lambda U_0^*=B_1\Lambda U_0^*,
\)
hence \(A_1=B_1\). If \(A_{p-1}=B_{p-1}\) for some \(2\le p\le N-1\), then
comparison of the \(p\)-th layer gives
\[
A_p\Lambda A_{p-1}^*
=
B_p\Lambda B_{p-1}^*
=
B_p\Lambda A_{p-1}^*,
\]
so \(A_p=B_p\). Induction proves injectivity.

It remains to show that \(\yy_X\) is an immersion. For
\(\mathbf V,\mathbf U\in K_{\beta}^{N-1}\), the orbit map satisfies
\[
\yy_X(\mathbf V\mathbf U)=\mathbf V\cdot \yy_X(\mathbf U),
\]
where multiplication on the left is componentwise and the dot denotes the
\(K_{\beta}^{N-1}\)-action on \(\Md{d}(\F)^N\). Hence it is enough to check the
differential at the identity.

Let
\(
\mathbf a=(a_{N-1},\dots,a_1)\in \mathfrak k_{\beta}^{N-1}.
\)
Differentiating
\[
\tau\longmapsto
\yy_X(\exp(\tau a_{N-1}),\dots,\exp(\tau a_1))
\]
at \(\tau=0\) gives the tangent vector with layer components
\begin{equation}\label{eq:general-tangent-formula}
\begin{aligned}
\bigl((\dd\yy_X)_{\mathbf e}(\mathbf a)\bigr)_N
&=-U_N\Lambda a_{N-1},\\
\bigl((\dd\yy_X)_{\mathbf e}(\mathbf a)\bigr)_p
&=a_p\Lambda-\Lambda a_{p-1},
\qquad 2\le p\le N-1,\\
\bigl((\dd\yy_X)_{\mathbf e}(\mathbf a)\bigr)_1
&=a_1\Lambda U_0^*.
\end{aligned}
\end{equation}
If \((\dd\yy_X)_{\mathbf e}(\mathbf a)=0\), then the first layer gives
\(
a_1\Lambda U_0^*=0,
\)
so \(a_1=0\). If \(a_{p-1}=0\) for some \(2\le p\le N-1\), then the \(p\)-th
layer gives
\(
a_p\Lambda=0,
\)
so \(a_p=0\). Hence \(a_1=\cdots=a_{N-1}=0\). Therefore
\((\dd\yy_X)_{\mathbf e}\) is injective, and \(\yy_X\) is an immersion.

Since \(K_{\beta}^{N-1}\) is compact, the injective immersion \(\yy_X\) is proper as
a map to \(\Md{d}(\F)^N\). A proper injective immersion is an embedding. By
surjectivity, its image is \(\OO_X^{\beta}\). Thus \(\OO_X^{\beta}\) is an embedded
submanifold of \(\Md{d}(\F)^N\), and \(\yy_X\) is a diffeomorphism onto
\(\OO_X^{\beta}\). Since \(K_{\beta}^{N-1}\) is compact, \(\OO_X^{\beta}\) is compact.

Finally, \(\yy_X(\mathbf e)= \mathbf C_X\), and the equivariance above gives
\[
\OO_X^{\beta}=\yy_X(K_{\beta}^{N-1})=K_{\beta}^{N-1}\cdot \mathbf C_X.
\]
\end{proof}

\begin{corollary}\label{cor:tangent-space-diagonal}
At the diagonal point \(\mathbf C=(\Lambda,\dots,\Lambda)\in \OO_{\Sigma}^{\beta}\), the
tangent space is
\[
T_{\mathbf C}\OO_{\Sigma}^{\beta}
=
\set{\mathbf c(\mathbf a):\mathbf a\in \mathfrak k_{\beta}^{N-1}},
\]
where \(\mathbf c(\mathbf a)\) has layer components
\begin{equation}\label{eq:diagonal-tangent-formula}
\begin{aligned}
\mathbf c(\mathbf a)_N
&=-\Lambda a_{N-1},\\
\mathbf c(\mathbf a)_p
&=a_p\Lambda-\Lambda a_{p-1},
\qquad 2\le p\le N-1,\\
\mathbf c(\mathbf a)_1
&=a_1\Lambda.
\end{aligned}
\end{equation}
\end{corollary}

\begin{proof}
Apply \Cref{prop:intro-orbit} to \(X=\Sigma\), so that \(U_N=U_0=\Id\) in
\eqref{eq:general-tangent-formula}. Since \(\yy_{\Sigma}\) is a diffeomorphism,
its differential at the identity is an isomorphism from \(\mathfrak k_{\beta}^{N-1}\) onto
\(T_{\mathbf C}\OO_{\Sigma}^{\beta}\).
\end{proof}

The map \(\yy_X\) parametrizes the compact orbit \(\OO_X^{\beta}\) inside the fixed
fibre \(\mathcal F_X^{\beta}\). Accordingly,
\eqref{eq:diagonal-tangent-formula} describes \(T_{\mathbf C}\OO_{\Sigma}^{\beta}\), and not
\(T_{\mathbf C}\MM^{\beta}\). Tangent directions along the full balanced manifold also include
directions in which the end-to-end matrix varies, equivalently variations of the
singular values and the endpoint singular-vector factors. For the analogous
separation between orbit directions in one fibre and directions in which the
end-to-end matrix varies in the complex DLN, see \cite[Chapter~5]{Kotwal2026}.

\section{The induced metric on the orbit}\label{sec:vertical}

By \Cref{prop:unitary-invariance}, the entropy depends only on the singular values of $X$. Hence, from now on we work at the diagonal point
\[
\mathbf C=(\Lambda,\dots,\Lambda)\in \OO_{\Sigma}^{\beta}.
\]
Write \(\Lambda=\diag(\lambda_1,\ldots,\lambda_d)\), so that
\(\lambda_s=\sigma_s^{1/N}\).
The tangent space at $\mathbf C$ is described by \Cref{cor:tangent-space-diagonal}. The next step is to choose an explicit real orthonormal basis of $\mathfrak{k}_{\beta}$ and compute the induced metric on the corresponding orbit directions.

\subsection*{A real orthonormal basis of the Lie algebra}

Set
\[
B_\beta=
\begin{cases}
\set{1}, & \F=\R,\\
\set{1,i}, & \F=\C,\\
\set{1,i,j,k}, & \F=\HH,
\end{cases}
\qquad
I_\beta=B_\beta\setminus\set{1}.
\]
Then $|B_\beta|=\beta$ and $|I_\beta|=\beta-1$. The set $B_\beta$ is an orthonormal basis of $\F$ over $\R$, so for $\eps,\eta\in B_\beta$ one has
\begin{equation}\label{eq:basis-units-orthonormal}
\re(\overline{\eps}\,\eta)=\delta_{\eps\eta}.
\end{equation}
For $1\le s<l\le d$ and $\eps\in B_\beta$, define
\begin{equation}\label{eq:offdiag-basis}
T_{sl}^{\eps}:=\frac{1}{\sqrt 2}\bigl(\eps E_{sl}-\overline{\eps}\,E_{ls}\bigr),
\end{equation}
and for $1\le s\le d$ and $\eta\in I_\beta$, define
\begin{equation}\label{eq:diag-basis}
D_s^{\eta}:=\eta E_{ss}.
\end{equation}
Each of these matrices is skew-Hermitian, hence lies in $\mathfrak{k}_{\beta}$. For $\F=\R$, only the matrices $T_{sl}^{1}$ occur. The matrices $D_s^{\eta}$ come from the imaginary diagonal part of $\mathfrak{k}_{\beta}$ and appear only in the complex and quaternionic cases.

\begin{lemma}\label{lem:lie-basis}
The set
\[
\set{T_{sl}^{\eps}:1\le s<l\le d,\ \eps\in B_\beta}
\cup
\set{D_s^{\eta}:1\le s\le d,\ \eta\in I_\beta}
\]
is a real orthonormal basis of $\mathfrak{k}_{\beta}$ for the inner product \eqref{eq:lie-inner-product}.
\end{lemma}

\begin{proof}
Orthogonality is a direct computation with matrix units. For $q,r\in\F$,
\[
(qE_{ab})(rE_{cd})=\delta_{bc}\,qr\,E_{ad}.
\]
Together with \eqref{eq:basis-units-orthonormal}, this gives
\[
\re\tr\bigl((T_{sl}^{\eps})^*T_{mn}^{\eta}\bigr)
=\delta_{sm}\delta_{ln}\delta_{\eps\eta},
\qquad
\re\tr\bigl((D_s^{\eta})^*D_m^{\theta}\bigr)
=\delta_{sm}\delta_{\eta\theta},
\]
while
\[
\re\tr\bigl((T_{sl}^{\eps})^*D_m^{\eta}\bigr)=0.
\]
Thus the displayed family is orthonormal.

It remains to check that the family spans $\mathfrak{k}_{\beta}$. A matrix $A\in \mathfrak{k}_{\beta}$ has arbitrary purely imaginary diagonal entries and off-diagonal entries satisfying $A_{ls}=-\overline{A_{sl}}$ for $s<l$. Expanding each off-diagonal entry $A_{sl}\in \F$ in the real basis $B_\beta$ and each diagonal entry in the real basis $I_\beta$ expresses $A$ uniquely as a real linear combination of the matrices \eqref{eq:offdiag-basis} and \eqref{eq:diag-basis}. Therefore they form a real orthonormal basis.
\end{proof}

For $1\le p\le N-1$, let $\mathbf e_{p,sl}^{\eps}\in \mathfrak{k}_{\beta}^{N-1}$ denote the element whose $p$th component is $T_{sl}^{\eps}$ and whose other components are zero. Likewise, let $\mathbf f_{p,s}^{\eta}\in \mathfrak{k}_{\beta}^{N-1}$ denote the element whose $p$th component is $D_s^{\eta}$ and whose other components are zero. By \Cref{lem:lie-basis}, the family
\[
\begin{aligned}
&\set{\mathbf e_{p,sl}^{\eps}:1\le p\le N-1,\ 1\le s<l\le d,\ \eps\in B_\beta}\\
&\hspace{5em}\cup\set{\mathbf f_{p,s}^{\eta}:1\le p\le N-1,\ 1\le s\le d,\ \eta\in I_\beta}
\end{aligned}
\]
is a real orthonormal basis of $\mathfrak{k}_{\beta}^{N-1}$.

We now push this basis to the orbit through \eqref{eq:diagonal-tangent-formula}. Set
\begin{equation}\label{eq:vertical-basis-vectors}
\mathbf v_{p,sl}^{\eps}:=\mathbf c(\mathbf e_{p,sl}^{\eps}),
\qquad
\mathbf w_{p,s}^{\eta}:=\mathbf c(\mathbf f_{p,s}^{\eta}).
\end{equation}
Since $\mathbf c:\mathfrak{k}_{\beta}^{N-1}\to T_{\mathbf C}\OO_{\Sigma}^{\beta}$ is an isomorphism by \Cref{cor:tangent-space-diagonal}, these vectors form a real basis of $T_{\mathbf C}\OO_{\Sigma}^{\beta}$.

Explicitly, for $1\le r\le N$, where $r$ denotes the layer index,
\[
(\mathbf v_{p,sl}^{\eps})_r=
\begin{cases}
-\Lambda T_{sl}^{\eps}, & r=p+1,\\
T_{sl}^{\eps}\Lambda, & r=p,\\
0, & \text{otherwise},
\end{cases}
\qquad
(\mathbf w_{p,s}^{\eta})_r=
\begin{cases}
-\Lambda D_s^{\eta}, & r=p+1,\\
D_s^{\eta}\Lambda, & r=p,\\
0, & \text{otherwise}.
\end{cases}
\]

\subsection*{Inner products and tridiagonal blocks}

Let $\iota_\Sigma$ denote the Riemannian metric on $\OO_{\Sigma}^{\beta}$ induced from the ambient inner product \eqref{eq:ambient-inner-product}.

\begin{lemma}\label{lem:basic-inner-products}
Among the vectors \eqref{eq:vertical-basis-vectors}, the following inner products,
together with those obtained by interchanging the two arguments, are the only
nonzero ones:
\begin{align}
\inner{\mathbf v_{p,sl}^{\eps}}{\mathbf v_{p,sl}^{\eps}}_{\iota_\Sigma}
&=\lambda_s^2+\lambda_l^2
&& (1\le p\le N-1),
&
\inner{\mathbf v_{p,sl}^{\eps}}{\mathbf v_{p+1,sl}^{\eps}}_{\iota_\Sigma}
&=-\lambda_s\lambda_l
&& (1\le p\le N-2),
\label{eq:offdiag-inner-products}\\
\inner{\mathbf w_{p,s}^{\eta}}{\mathbf w_{p,s}^{\eta}}_{\iota_\Sigma}
&=2\lambda_s^2
&& (1\le p\le N-1),
&
\inner{\mathbf w_{p,s}^{\eta}}{\mathbf w_{p+1,s}^{\eta}}_{\iota_\Sigma}
&=-\lambda_s^2
&& (1\le p\le N-2).
\label{eq:diag-inner-products}
\end{align}
Any inner product that is neither listed in
\eqref{eq:offdiag-inner-products}--\eqref{eq:diag-inner-products}
nor obtained from one of these by interchanging the two arguments vanishes.
\end{lemma}

\begin{proof}
Let \(p\) and \(q\) denote the depth labels of the two pushed-forward
basis vectors. The layer index itself will be denoted by \(r\). If
\(|p-q|>1\), then the two pushed-forward vectors have disjoint supports
as functions of the layer index \(r\), so their inner product is zero.
It is therefore enough to consider the cases \(q=p\) and \(q=p+1\);
the case \(q=p-1\) follows by symmetry.

For the off-diagonal vectors, a direct computation from \eqref{eq:offdiag-basis} gives
\[
\Lambda T_{sl}^{\eps}=\frac{1}{\sqrt 2}\bigl(\lambda_s\eps E_{sl}-\lambda_l\overline{\eps}\,E_{ls}\bigr),
\qquad
T_{sl}^{\eps}\Lambda=\frac{1}{\sqrt 2}\bigl(\lambda_l\eps E_{sl}-\lambda_s\overline{\eps}\,E_{ls}\bigr).
\]
Using the matrix-unit identities and \eqref{eq:basis-units-orthonormal}, we obtain
\begin{align*}
\re\tr\bigl((\Lambda T_{sl}^{\eps})^*(\Lambda T_{mn}^{\eta})\bigr)
&=\frac12(\lambda_s^2+\lambda_l^2)\,\delta_{sm}\delta_{ln}\delta_{\eps\eta},\\
\re\tr\bigl((T_{sl}^{\eps}\Lambda)^*(T_{mn}^{\eta}\Lambda)\bigr)
&=\frac12(\lambda_s^2+\lambda_l^2)\,\delta_{sm}\delta_{ln}\delta_{\eps\eta},\\
\re\tr\bigl((\Lambda T_{sl}^{\eps})^*(T_{mn}^{\eta}\Lambda)\bigr)
&=\lambda_s\lambda_l\,\delta_{sm}\delta_{ln}\delta_{\eps\eta}.
\end{align*}
Therefore, for $1\le p\le N-1$,
\[
\inner{\mathbf v_{p,sl}^{\eps}}{\mathbf v_{p,mn}^{\eta}}_{\iota_\Sigma}
=(\lambda_s^2+\lambda_l^2)\,\delta_{sm}\delta_{ln}\delta_{\eps\eta}.
\]
For adjacent layers, with $1\le p\le N-2$, the only common layer is $p+1$, and the sign comes from the component $-\Lambda T_{sl}^{\eps}$ in $\mathbf v_{p,sl}^{\eps}$. Hence
\[
\inner{\mathbf v_{p,sl}^{\eps}}{\mathbf v_{p+1,mn}^{\eta}}_{\iota_\Sigma}
=
-\re\tr\bigl((\Lambda T_{sl}^{\eps})^*(T_{mn}^{\eta}\Lambda)\bigr)
=
-\lambda_s\lambda_l\,\delta_{sm}\delta_{ln}\delta_{\eps\eta}.
\]
This proves \eqref{eq:offdiag-inner-products} and shows that distinct pairs $(s,l)$ or distinct basis units are orthogonal.

For the imaginary diagonal vectors,
\[
\Lambda D_s^{\eta}=D_s^{\eta}\Lambda=\lambda_s\eta E_{ss}.
\]
Thus
\begin{align*}
\re\tr\bigl((\Lambda D_s^{\eta})^*(\Lambda D_m^{\theta})\bigr)
&=\lambda_s^2\delta_{sm}\delta_{\eta\theta},\\
\re\tr\bigl((D_s^{\eta}\Lambda)^*(D_m^{\theta}\Lambda)\bigr)
&=\lambda_s^2\delta_{sm}\delta_{\eta\theta},\\
\re\tr\bigl((\Lambda D_s^{\eta})^*(D_m^{\theta}\Lambda)\bigr)
&=\lambda_s^2\delta_{sm}\delta_{\eta\theta}.
\end{align*}
Therefore, for $1\le p\le N-1$,
\[
\inner{\mathbf w_{p,s}^{\eta}}{\mathbf w_{p,m}^{\theta}}_{\iota_\Sigma}
=2\lambda_s^2\delta_{sm}\delta_{\eta\theta}.
\]
For adjacent layers, with $1\le p\le N-2$, the only common layer is $p+1$, and the sign again comes from the component $-\Lambda D_s^{\eta}$ in $\mathbf w_{p,s}^{\eta}$. Hence
\[
\inner{\mathbf w_{p,s}^{\eta}}{\mathbf w_{p+1,m}^{\theta}}_{\iota_\Sigma}
=
-\re\tr\bigl((\Lambda D_s^{\eta})^*(D_m^{\theta}\Lambda)\bigr)
=
-\lambda_s^2\delta_{sm}\delta_{\eta\theta}.
\]
This proves \eqref{eq:diag-inner-products}. Every pairing between an off-diagonal vector and a diagonal vector vanishes because the relevant products of matrix units have zero trace.
\end{proof}

Introduce the standard tridiagonal matrix
\begin{equation}\label{eq:tridiagonal-L}
L_{N-1}:=
\begin{pmatrix}
2&-1&&&0\\
-1&2&-1&&\\
&\ddots&\ddots&\ddots&\\
&&-1&2&-1\\
0&&&-1&2
\end{pmatrix}\in \Md{N-1}(\R).
\end{equation}
For $1\le s<l\le d$, define
\begin{equation}\label{eq:offdiag-block}
H_{sl}:=(\lambda_s-\lambda_l)^2\Id_{N-1}+\lambda_s\lambda_lL_{N-1}.
\end{equation}
Then $H_{sl}$ is the $(N-1)\times(N-1)$ tridiagonal matrix with diagonal entries $\lambda_s^2+\lambda_l^2$ and nearest off-diagonal entries $-\lambda_s\lambda_l$.

Let $G_{\Sigma,\beta}$ be the matrix of inner products of $\iota_\Sigma$ on
$T_{\mathbf C}\OO_{\Sigma}^{\beta}$ in the ordered pushed-forward basis obtained as follows:
group together the vectors $\mathbf v_{p,sl}^{\eps}$ with fixed
$(s,l,\eps)$ and order them by depth $p=1,\dots,N-1$, and group together the
vectors $\mathbf w_{p,s}^{\eta}$ with fixed $(s,\eta)$ and again order them by
depth.

\begin{proposition}\label{prop:gram-blocks}
The matrix $G_{\Sigma,\beta}$ is block diagonal. More precisely,
\begin{itemize}[leftmargin=2em]
    \item for each pair $1\le s<l\le d$ and each $\eps\in B_\beta$, the corresponding block is $H_{sl}$;
    \item for each $1\le s\le d$ and each $\eta\in I_\beta$, the corresponding block is $\lambda_s^2L_{N-1}$.
\end{itemize}
In particular,
\begin{equation}\label{eq:gram-determinant-factorization}
\det G_{\Sigma,\beta}
=\prod_{1\le s<l\le d}\det(H_{sl})^{\beta}
\prod_{s=1}^d\det(\lambda_s^2L_{N-1})^{\beta-1}.
\end{equation}
\end{proposition}

\begin{proof}
By \Cref{lem:basic-inner-products}, vectors with different labels are orthogonal. For fixed $(s,l,\eps)$, the vectors
\[
\mathbf v_{1,sl}^{\eps},\dots,\mathbf v_{N-1,sl}^{\eps}
\]
have the tridiagonal inner products described in \eqref{eq:offdiag-inner-products}, so their block is exactly $H_{sl}$. For fixed $(s,\eta)$, the vectors
\[
\mathbf w_{1,s}^{\eta},\dots,\mathbf w_{N-1,s}^{\eta}
\]
have the tridiagonal inner products described in \eqref{eq:diag-inner-products}, so their block is $\lambda_s^2L_{N-1}$. This proves the block decomposition.

The multiplicities are immediate: there are $|B_\beta|=\beta$ off-diagonal copies for each pair $1\le s<l\le d$ and $|I_\beta|=\beta-1$ diagonal copies for each $1\le s\le d$. Taking determinants of the blocks gives \eqref{eq:gram-determinant-factorization}.
\end{proof}

For $\F=\R$, only one copy of each off-diagonal block $H_{sl}$ appears. For $\F=\C$ or $\HH$, the off-diagonal blocks occur with multiplicity $\beta$, and the imaginary diagonal directions contribute the additional $\beta-1$ copies of the blocks $\lambda_s^2L_{N-1}$.
\section{Entropy and a determinant identity}\label{sec:entropy}
\subsection*{Determinants of the tridiagonal blocks}

\Cref{prop:gram-blocks} reduces the volume computation to the
determinants of the tridiagonal blocks \(L_{N-1}\) and \(H_{sl}\). We now evaluate
these determinants.

\begin{lemma}\label{lem:det-laplacian}
For every \(n\ge 1\),
\[
\det(L_n)=n+1,
\]
where \(L_n\) is the \(n\times n\) tridiagonal matrix with diagonal entries \(2\)
and nearest off-diagonal entries \(-1\).
\end{lemma}

\begin{proof}
Set \(\ell_n=\det(L_n)\). Expanding along the first row gives the recurrence
\[
\ell_n=2\ell_{n-1}-\ell_{n-2},
\qquad n\ge 3,
\]
with initial values \(\ell_1=2\) and \(\ell_2=3\). The sequence
\(\ell_n=n+1\) satisfies the same recurrence and initial conditions, so
\(\ell_n=n+1\) for all \(n\).
\end{proof}

\begin{proposition}\label{prop:det-offdiag-block}
For \(1\le s<l\le d\),
\begin{equation}\label{eq:offdiag-determinant-sum}
\det(H_{sl})
=
\sum_{m=0}^{N-1}
\lambda_s^{2(N-1-m)}\lambda_l^{2m}.
\end{equation}
In particular, if \(\lambda_s\ne \lambda_l\), then
\begin{equation}\label{eq:offdiag-determinant}
\det(H_{sl})
=
\frac{\lambda_s^{2N}-\lambda_l^{2N}}{\lambda_s^2-\lambda_l^2}
=
\frac{\sigma_s^2-\sigma_l^2}
     {\sigma_s^{2/N}-\sigma_l^{2/N}}.
\end{equation}
Moreover,
\begin{equation}\label{eq:diag-determinant}
\det(\lambda_s^2L_{N-1})=N\lambda_s^{2(N-1)}=N\sigma_s^{2-2/N}.
\end{equation}
\end{proposition}

\begin{proof}
Equation~\eqref{eq:diag-determinant} follows immediately from
\Cref{lem:det-laplacian}:
\[
\det(\lambda_s^2L_{N-1})
=
\lambda_s^{2(N-1)}\det(L_{N-1})
=
N\lambda_s^{2(N-1)}.
\]
Since \(\lambda_s^N=\sigma_s\), this is exactly
\eqref{eq:diag-determinant}.
For the block \(H_{sl}\), let \(D_n\) be the determinant of the
\(n\times n\) tridiagonal Toeplitz matrix with diagonal entry
\(\lambda_s^2+\lambda_l^2\) and nearest off-diagonal entry
\(-\lambda_s\lambda_l\). Then \(D_{N-1}=\det(H_{sl})\). Expanding along the
first row gives
\[
D_n
=
(\lambda_s^2+\lambda_l^2)D_{n-1}
-\lambda_s^2\lambda_l^2D_{n-2},
\qquad n\ge 2,
\]
with initial conditions \(D_0=1\) and
\(D_1=\lambda_s^2+\lambda_l^2\).

Set
\[
S_n=\sum_{m=0}^{n}\lambda_s^{2(n-m)}\lambda_l^{2m}.
\]
The sequence \(S_n\) has the same recurrence and the same initial values as
\(D_n\). Hence \(D_n=S_n\) for all \(n\), and setting \(n=N-1\) gives
\eqref{eq:offdiag-determinant-sum}. If \(\lambda_s\ne\lambda_l\), the finite
geometric sum gives
\[
\sum_{m=0}^{N-1}
\lambda_s^{2(N-1-m)}\lambda_l^{2m}
=
\frac{\lambda_s^{2N}-\lambda_l^{2N}}
     {\lambda_s^2-\lambda_l^2}.
\]
Using \(\lambda_r^N=\sigma_r\) gives the second quotient in
\eqref{eq:offdiag-determinant}.
\end{proof}

\begin{corollary}\label{cor:total-gram-determinant}
Let \(G_{\Sigma,\beta}\) be the matrix of inner products of the induced metric on
\(T_{\mathbf C}\OO_{\Sigma}^{\beta}\) in the pushed-forward basis ordered as in
\Cref{prop:gram-blocks}. Assume \(\sigma_1>\cdots>\sigma_d>0\). Then
\begin{equation}\label{eq:total-gram-determinant-vandermonde}
\det G_{\Sigma,\beta}
=
N^{(\beta-1)d}
(\det\Sigma)^{2(\beta-1)(1-1/N)}
\left(
\frac{\van(\Sigma^2)}{\van(\Sigma^{2/N})}
\right)^{\beta}.
\end{equation}
\end{corollary}

\begin{proof}
Insert \eqref{eq:offdiag-determinant} and \eqref{eq:diag-determinant} into the
block product formula \eqref{eq:gram-determinant-factorization}. The diagonal
blocks contribute
\[
\prod_{s=1}^d \det(\lambda_s^2L_{N-1})^{\beta-1}
=
\prod_{s=1}^d
\left(N\lambda_s^{2(N-1)}\right)^{\beta-1}
=
N^{(\beta-1)d}(\det\Lambda)^{2(\beta-1)(N-1)}.
\]
The off-diagonal blocks contribute
\[
\prod_{1\le s<l\le d}\det(H_{sl})^{\beta}
=
\prod_{1\le s<l\le d}
\left(
\frac{\sigma_s^2-\sigma_l^2}
     {\sigma_s^{2/N}-\sigma_l^{2/N}}
\right)^{\beta}.
\]
Multiplying the two contributions gives
\[
\det G_{\Sigma,\beta}
=
N^{(\beta-1)d}
(\det\Lambda)^{2(\beta-1)(N-1)}
\prod_{1\le s<l\le d}
\left(
\frac{\sigma_s^2-\sigma_l^2}
     {\sigma_s^{2/N}-\sigma_l^{2/N}}
\right)^{\beta}.
\]
Finally,
\[
(\det\Lambda)^{2(\beta-1)(N-1)}
=
(\det\Sigma)^{2(\beta-1)(1-1/N)}
\]
and
\[
\prod_{1\le s<l\le d}
\left(
\frac{\sigma_s^2-\sigma_l^2}
     {\sigma_s^{2/N}-\sigma_l^{2/N}}
\right)^{\beta}
=
\left(
\frac{\van(\Sigma^2)}{\van(\Sigma^{2/N})}
\right)^{\beta}.
\]
Thus \eqref{eq:total-gram-determinant-vandermonde} follows.
\end{proof}

\subsection*{Volume and entropy}

We now pull the induced metric on the orbit back to \(K_\beta^{N-1}\) by the
orbit map.

\begin{proposition}\label{prop:orbit-volume-pullback}
Assume \(X\in \mathrm{GL}_d(\F)\) has distinct singular values. Let \(\iota_X\)
denote the metric on \(\OO_X^\beta\) induced by the ambient metric, and set
\[
\gamma_X:=\yy_X^*(\iota_X),
\]
where \(\yy_X\) is the orbit map of \Cref{prop:intro-orbit}. Then
\(\gamma_X\) is left-invariant on \(K_\beta^{N-1}\). Consequently,
\begin{equation}\label{eq:orbit-volume-from-gram}
\operatorname{vol}(\OO_X^\beta)
=
c_\beta^{N-1}\det(G_{X,\beta})^{1/2},
\end{equation}
where \(G_{X,\beta}\) is the matrix of \(\gamma_X\) at the identity in any
orthonormal basis of \(\mathfrak{k}_\beta^{N-1}\).
\end{proposition}

\begin{proof}
For \(\mathbf{V},\mathbf{U}\in K_\beta^{N-1}\), the orbit map satisfies
\(
\yy_X(\mathbf{V}\mathbf{U})=\mathbf{V}\cdot \yy_X(\mathbf{U}).
\)
The action of \(K_\beta^{N-1}\) on the orbit is by isometries for the induced
metric. If \(\ell_{\mathbf{V}}\) denotes left translation by \(\mathbf{V}\), then
\[
(\ell_{\mathbf{V}})^*\gamma_X
=
(\yy_X\circ \ell_{\mathbf{V}})^*\iota_X
=
(\mathbf{V}\cdot \yy_X)^*\iota_X
=
\yy_X^*(\mathbf{V}\cdot)^*\iota_X
=
\yy_X^*\iota_X
=
\gamma_X.
\]
Thus \(\gamma_X\) is left-invariant.

The reference volume on \(K_\beta^{N-1}\) is the product volume induced by
\(
\langle A,B\rangle_{\mathfrak{k}}
=
\operatorname{Re}\operatorname{Tr}(A^*B)
\)
on each factor \(\mathfrak{k}_\beta\). Therefore
\(
\operatorname{vol}(K_\beta^{N-1})=c_\beta^{N-1}.
\)
When \(\F=\R\), this is the volume of all of \(O_d\). Choose an orthonormal basis
of \(\mathfrak{k}_\beta^{N-1}\) for this product metric and extend it by left
translation. In this frame, the reference metric has matrix \(I\), while
\(\gamma_X\) has the constant matrix \(G_{X,\beta}\). Hence the Riemannian
density of \(\gamma_X\) is \(\det(G_{X,\beta})^{1/2}\) times the reference density.
Integrating over \(K_\beta^{N-1}\) gives \eqref{eq:orbit-volume-from-gram}.
\end{proof}

\begin{proof}[Proof of \Cref{thm:intro-entropy}]
By the endpoint isometry, the entropy depends only on the singular values of
\(X\). Thus
\(
S^\beta(X)=S^\beta(\Sigma).
\)
By \Cref{prop:orbit-volume-pullback} and \Cref{cor:total-gram-determinant},
\[
\operatorname{vol}(\OO_\Sigma^\beta)
=
c_\beta^{N-1}
N^{(\beta-1)d/2}
(\det\Sigma)^{(\beta-1)(1-1/N)}
\left(
\frac{\van(\Sigma^2)}{\van(\Sigma^{2/N})}
\right)^{\beta/2}.
\]
Taking logarithms gives
\[
S^\beta(X)
=
(N-1)\log c_\beta
+
\frac{\beta}{2}
\log\left(
\frac{\van(\Sigma^2)}{\van(\Sigma^{2/N})}
\right)
+
(\beta-1)
\left(
\frac{d}{2}\log N
+
\log\left(
\frac{\det\Sigma}{\det(\Sigma^{1/N})}
\right)
\right).
\]
Since
\[
\log\left(
\frac{\det\Sigma}{\det(\Sigma^{1/N})}
\right)
=
\left(1-\frac{1}{N}\right)
\sum_{s=1}^d \log\sigma_s,
\]
this is equivalent to the sum formula.
\end{proof}

\subsection*{A determinant identity}

\begin{lemma}\label{lem:det-real-AN-inverse}
Let \(X=U_N\Sigma U_0^*\) be a singular value decomposition over \(\F\). Assume
\(\sigma_1>\cdots>\sigma_d>0\). Then
\begin{equation}\label{eq:det-real-AN-inverse}
\det_{\R}(\mathcal A_{N,X}^{-1})
=
N^{-\beta d}
(\det\Sigma)^{-2\beta(1-1/N)}
\left(
\frac{\van(\Sigma^{2/N})}{\van(\Sigma^2)}
\right)^{2\beta}.
\end{equation}
\end{lemma}

\begin{proof}
Let \(u_s\) and \(v_l\) be the columns of \(U_N\) and \(U_0\). The vectors
\[
u_s\varepsilon v_l^*,
\qquad
1\le s,l\le d,\quad \varepsilon\in B_\beta,
\]
form a real orthonormal basis of \(M_d(\F)\). Indeed,
\[
\left\langle
u_s\varepsilon v_l^*,
u_m\eta v_n^*
\right\rangle
=
\delta_{sm}\delta_{ln}\operatorname{Re}(\overline{\varepsilon}\eta)
=
\delta_{sm}\delta_{ln}\delta_{\varepsilon\eta}.
\]
For this basis,
\[
\mathcal A_{N,X}(u_s\varepsilon v_l^*)
=
\alpha_{sl}u_s\varepsilon v_l^*,
\]
where
\[
\alpha_{sl}
:=
\sum_{p=1}^N
\sigma_s^{2(N-p)/N}
\sigma_l^{2(p-1)/N}.
\]
The eigenvalue does not depend on \(\varepsilon\), so each \(\alpha_{sl}\) has real
multiplicity \(\beta\). For \(s=l\),
\(
\alpha_{ss}=N\sigma_s^{2-2/N}.
\)
For \(s\ne l\), the identity for a geometric series gives
\[
\alpha_{sl}
=
\frac{\sigma_s^2-\sigma_l^2}
     {\sigma_s^{2/N}-\sigma_l^{2/N}}.
\]
Hence
\[
\det_{\R}(\mathcal A_{N,X}^{-1})
=
\prod_{s=1}^d \alpha_{ss}^{-\beta}
\prod_{1\le s<l\le d}\alpha_{sl}^{-2\beta}.
\]
Substituting the formulas for \(\alpha_{ss}\) and \(\alpha_{sl}\) gives
\[
\det_{\R}(\mathcal A_{N,X}^{-1})
=
N^{-\beta d}
(\det\Sigma)^{-2\beta(1-1/N)}
\prod_{1\le s<l\le d}
\left(
\frac{\sigma_s^{2/N}-\sigma_l^{2/N}}
     {\sigma_s^2-\sigma_l^2}
\right)^{2\beta},
\]
which is \eqref{eq:det-real-AN-inverse}.
\end{proof}

\begin{proof}[Proof of \Cref{prop:intro-operator}]
By the endpoint isometry, it suffices to use the diagonal representative \(\Sigma\).
Combining \eqref{eq:orbit-volume-from-gram} with
\eqref{eq:total-gram-determinant-vandermonde} gives
\[
\operatorname{vol}(\OO_X^\beta)
=
c_\beta^{N-1}
N^{(\beta-1)d/2}
(\det\Sigma)^{(\beta-1)(1-1/N)}
\left(
\frac{\van(\Sigma^2)}{\van(\Sigma^{2/N})}
\right)^{\beta/2}.
\]
Taking the fourth root of \eqref{eq:det-real-AN-inverse} gives
\[
\det_{\R}(\mathcal A_{N,X}^{-1})^{1/4}
=
N^{-\beta d/4}
(\det\Sigma)^{-\beta(1-1/N)/2}
\left(
\frac{\van(\Sigma^{2/N})}{\van(\Sigma^2)}
\right)^{\beta/2}.
\]
The Vandermonde factors cancel. The powers of \(N\) combine as
\[
\frac{(\beta-1)d}{2}-\frac{\beta d}{4}
=
\frac{(\beta-2)d}{4},
\]
and the powers of \(\det\Sigma\) combine as
\[
(\beta-1)\left(1-\frac{1}{N}\right)
-
\frac{\beta}{2}\left(1-\frac{1}{N}\right)
=
\frac{\beta-2}{2}\left(1-\frac{1}{N}\right).
\]
Since
\[
(\det\Sigma)^{1-1/N}
=
\frac{\det\Sigma}{\det(\Sigma^{1/N})},
\]
we obtain
\[
\operatorname{vol}(\OO_X^\beta)
\det_{\R}(\mathcal A_{N,X}^{-1})^{1/4}
=
c_\beta^{N-1}
N^{(\beta-2)d/4}
\left(
\frac{\det\Sigma}{\det(\Sigma^{1/N})}
\right)^{(\beta-2)/2}.
\]
Taking logarithms gives the equivalent entropy identity, because \(\mathcal A_{N,X}\) is
positive.
\end{proof}

\section*{Acknowledgements}

The authors are grateful to Govind Menon and Tianmin Yu for helpful discussions and insightful remarks that have improved this work. This work was partially supported by NSF grant DMS 2407055.

\printbibliography

\end{document}